\documentclass[twoside,11pt]{article}
\usepackage{jmlr2e}
\usepackage[latin9]{inputenc}
\usepackage{amsmath}
\usepackage{amssymb}
\usepackage{graphicx}
\usepackage{comment}
\usepackage{enumerate}
\usepackage{algorithm}
\usepackage{algorithmic}

\usepackage{amsfonts}
\DeclareGraphicsExtensions{.pdf,.jpg,.png}

\usepackage{caption}
\usepackage{lipsum}
\usepackage{xfrac}
\usepackage{array}
\usepackage{url}
\usepackage{graphicx}
\newcommand{\indep}{\rotatebox[origin=c]{90}{$\models$}}
\RequirePackage{pdflscape}
\newcolumntype{M}[1]{>{\centering\arraybackslash}m{#1}}
\newcolumntype{N}{@{}m{0pt}@{}}
\ShortHeadings{Gaussian Lower Bound for the Information Bottleneck Limit}{Painsky and Tishby} \firstpageno{1}

\begin{document}

\title{Gaussian Lower Bound for the Information Bottleneck Limit}

\author{\name Amichai Painsky \email amichai.painsky@mail.huji.ac.il 
       \AND
       \name Naftali Tishby \email tishby@cs.huji.ac.il \\
        \addr School of Computer Science and Engineering and\\
       The Interdisciplinary Center for Neural Computation\\
       The Hebrew University of Jerusalem\\
	Givat Ram, Jerusalem 91904, Israel}

\editor{TBD}

\maketitle

\begin{abstract}
The Information Bottleneck (IB) is a conceptual method for extracting the most compact, yet informative, representation of a set of variables, with respect to the target. It generalizes the notion of minimal sufficient statistics from classical parametric statistics to a broader information-theoretic sense. The IB curve defines the optimal trade-off between representation complexity and its predictive power. Specifically, it is achieved by minimizing the level of mutual information (MI) between the representation and the original variables, subject to a minimal level of MI between the representation and the target. This problem is shown to be in general NP hard. One important exception is the multivariate Gaussian case, for which the Gaussian IB (GIB) is known to obtain an analytical closed form solution, similar to Canonical Correlation Analysis (CCA). In this work we introduce a Gaussian lower bound to the IB curve; we find an embedding of the data which maximizes its ``Gaussian part", on which we apply the GIB.  This embedding provides an efficient (and practical) representation of any arbitrary data-set (in the IB sense), which in addition holds the favorable properties of a Gaussian distribution. Importantly, we show that the optimal Gaussian embedding is bounded from above by non-linear CCA.  This allows a fundamental limit for our ability to Gaussianize arbitrary data-sets and solve complex problems by linear methods.

\end{abstract}


\section{Introduction}
\label{intro}
The problem of extracting the relevant aspects of complex data is a long standing staple in statistics and machine learning. The Information Bottleneck (IB) method, presented by \cite{tishby1999information}, approaches this problem by extending its classical notion to a broader information-theoretic setup. Specifically, given the joint distribution of a set of explanatory variables $\underline{X}$ and a target variable $\underline{Y}$ (which may also be of a higher dimension), the IB method strives to find the most compressed representation of $\underline{X}$, while preserving information about $\underline{Y}$. Thus, $\underline{Y}$ implicitly regulates the compression of $\underline{X}$, so that its compressed representation maintains a level of relevance as explanatory variables with regards to $\underline{Y}$. The IB problem is formally defined as follows:   

\begin{equation} \label{IB problem}
\begin{aligned}
& {\min_{P(\underline{T}|\underline{X})}}
& & I(\underline{X};\underline{T}) \\
& \text{subject to}
& & I(\underline{T};\underline{Y}) \geq I_Y
\end{aligned}
\end{equation}
where $\underline{T}$ is the compressed representation of $\underline{X}$ and the minimization is over the mapping of $\underline{X}$ to $\underline{T}$, defined by the conditional probability $P(\underline{T}|\underline{X})$. Here, $I_Y$ is a constant parameter that sets the level of information to be preserved between the compressed representation and the target. Solving this problem for a range of $I_Y$ values defines the \textit{IB curve} -- a continuous concave curve which demonstrates the optimal trade-off between representation complexity (regarded as $I(\underline{X};\underline{T})$) and predictive power ($I(\underline{T};\underline{Y})$). 

The IB method showed to be a powerful tool in a variety of machine learning domains and related areas \citep{slonim2000document,friedman2001multivariate,sinkkonen2002clustering,slonim2005information,hecht2009speaker}.  
It is also applicable to other fields such as neuroscience \citep{schneidman2001analyzing} and optimal control \citep{tishby2011information}.  Recently, \cite{tishby2015deep} and \cite{shwartz2017opening} demonstrated its abilities in analyzing and optimizing the performance of deep neural networks.  

Generally speaking, solving the IB problem (\ref{IB problem}) for an arbitrary joint distribution is not a simple task. In the introduction of the IB method, \cite{tishby1999information} defined a set of self-consistent equations which formulate the necessary conditions for the optimal solution of (\ref{IB problem}). Further, they provide an iterative Arimoto-- Blahut like algorithm which shows to converge to local optimum. In general, these equations do not hold a tractable solution and are usually approximated by different means \citep{slonim2002information}. An extensive attention was given to the simpler categorical setup, where the IB curve is somewhat easier to approximate. Here, $\underline{X}$ and $\underline{Y}$ take values on a finite set and $\underline{T}$ represents (soft and informative) clusters of $\underline{X}$ (REF). Naturally, the IB problem also applies for continuous variables. In this case, approximating the solution to the self-consistent equations is even more involved. A special exception is the Gaussian case, where $\underline{X}$ and $\underline{Y}$ are assumed to follow a jointly normal distribution and the \textit{Gaussian IB}  problem (GIB) is analytically solved by linear projections to the canonical correlation vector space \citep{chechik2005information}. However, evaluating the IB curve for arbitrary continuous random variables is still considered a highly complicated task where most attempts focus on approximating or bounding it \citep{rey2012meta,chalk2016relevant}. A detailed discussion regarding currently known methods is provided in the following section.

In this work we present a novel Gaussian lower bound to the IB curve, which applies to all types of random variables (continuous, nominal and categorical). Our bound strives to maximize the ``jointly Gaussian part" of the data and apply the analytical GIB to it. Specifically, we seek for two transformations, $\underline{U}=\phi(\underline{X})$ and $\underline{V}=\psi(\underline{Y})$ so that $\underline{U}$ and $\underline{V}$ are highly correlated and ``as jointly Gaussian as possible". In addition, we ask that the transformations preserve as much information as possible between $\underline{X}$ and $\underline{Y}$. This way, we maximize the portion of the data that can be explained by linear means, $I(\underline{U};\underline{V})\leq I(\underline{X};\underline{Y})$, specifically using the GIB. 

In fact, our results goes beyond the specific context of the information bottleneck. In this work we tackle the fundamental question of linearizing non-linear problems. Specifically, we ask ourselves whether it is possible to ``push" all the information in the data to its second moments. This problem has received a great amount of attention over the years. For example, \cite{schneidman2006weak} discuss this problem in the context of neural networks; they provide preliminary evidence that in the vertebrate retina, weak pairwise correlations may describe the collective (non-linear) behavior of neurons. In this work, we provide both fundamental limits and constructive algorithms for maximizing the part of the data that can be optimally analyzed by linear means. This basic property holds both theoretical and practical implications, as it defines the maximal portion which allows favorable analytical properties in many applications. Interestingly, we show that even if we allow the transformations $\underline{U}=\phi(\underline{X})$ and $\underline{V}=\psi(\underline{Y})$ to increase the dimensions of $\underline{X}$ and $\underline{Y}$, our ability to linearize the problem is still limited, and governed by the non-linear canonical correlations \citep{breiman1985estimating} of the original variables.% As demonstrated in the following sections, our suggested bound may require some extensive computation, especially in the high dimensional case. Hence, we further provide upper and lower bounds for our suggested approach, which are much easier to attain. In addition, these bounds reveal an important and useful connection between mutual information and non-linear canonical correlations. 

Our suggested approach may also be viewed as an extension of the \textit{Shannon lower bound} \citep{cover2012elements}, for evaluating the mutual information. In his seminal work, Shannon provided an analytical Gaussian lower bound for the generally involved rate distortion function. He showed that the rate distortion function $R(D)$ can be bounded from below by $h(X)-\frac{1}{2}\log(2\pi e D)$ where $X$ is the compressed source, $h(X)$ is its corresponding deferential entropy and $\frac{1}{2}\log(2\pi e D)$ is the differential entropy of an independent Gaussian noise with a maximal distortion level $D$. This bound holds some favorable theoretical properties  \citep{cover2012elements} and serves as one of the most basic tools for approximating the  rate distortion function to this very day. In this work, we use a somewhat similar rationale and derive a Gaussian lower bound for the mutual information of two random variables, which holds an analytical expression just like the Shannon's bound. We then extend our result to the entire IB curve and discuss its theoretical properties and practical considerations.

The rest of this manuscript is organized as follows: In Section \ref{Related work} we review previous work on the IB method for continuous random variables. Section \ref{Problem Formulation} defines our suggested lower bound and formulates it as an optimization problem. We then propose a set of solutions and bounds to this problem, as we distinguish between the easier univariate case (Section \ref{The 1-D case}) and the more involved multivariate case (Section \ref{multivariate case}). Finally, in Section \ref{Gaussian lower bound for the Information Bottleneck Curve} we extend our results to the entire IB curve.

\section{Related work} 
\label{Related work}
As discussed in the previous section, solving the IB problem for continuous variables is in general a difficult task. A special exception is where $\underline{X}$ and $\underline{Y}$ follow a jointly normal distribution.  \cite{chechik2005information}  show that in this case, the Gaussian IB problem (GIB) is solved by a noisy linear projection, $T=A\underline{X} + \underline{\zeta}$. Specifically, assume that $\underline{X}$ and $\underline{Y}$ are of dimensions $n_X$ and $n_Y$ respectively and denote the covariance matrix of $\underline{X}$ as $C_{\underline{X}}$ while the conditional covariance matrix of $\underline{X}|\underline{Y}$ is  $C_{\underline{X}|\underline{Y}}$. Then,  $\underline{\zeta}$ is a Gaussian random vector with a zero mean and a unit covariance matrix, independent of $\underline{X}$. The matrix $A$ is defined as follows: 

\begin{equation}
A=\left\{ \begin{array}{cc}
[0^T; \dots; 0^T] & 0 \leq \beta \leq \beta_1^C\\
{}[a_1v_1^T ; 0^T; \dots ; 0^T] & \beta_1^C \leq \beta \leq \beta_2^C\\
{}[a_1v_1^T ; a_2v_2^T; 0^T; \dots ; 0^T] & \beta_2^C \leq \beta \leq \beta_3^C\\
\vdots & \vdots
\end{array}\right\} .
\end{equation} 
where $\{v_1^T,v_2^T,\dots,v_{n_x}^T \}$ are the left eigenvectors of $C_{\underline{X}|\underline{Y}} C_{\underline{X}}^{-1}$, sorted by
their corresponding ascending eigenvalues $\lambda_1, \dots , \lambda_{n_X}$, $\beta_I^C=\frac{1}{\lambda_i}$ are the critical $\beta$ values, $a_i$ are defined by $a_i=\sqrt{\frac{\beta(1-\lambda_i)-1}{\lambda_i r_i}} $. $r_i=v_i^T C_{\underline{X}} v_i$ and $0^T$ is an $n_X$ row vector of zeros. Notice that the critical values $\beta$ correspond to the slope of the IB curve, as they represent the Lagrange multipliers of the IB problem. 

Unfortunately, this solution is limited to jointly Gaussian random variables. In fact, it can be shown that a closed form analytical solution (for continuous random variables) may only exist under quite restrictive assumptions on the underlaying distribution. Moreover, as the IB curve is so challenging to evaluate in the general case, most known attempts either focus on extending the GIB to other distributions under varying assumptions, or approximate the IB curve by different means.

\cite{rey2012meta} reformulate the IB problem in terms of probabilistic copulas. They show that under a Gaussian copula assumption, an analytical solution (which extends the GIB) applies to joint distributions with arbitrary marginals. This formulation provides several interesting insights on the IB problem. However, its practical implications are quite limited as the Gaussian copula assumption is very restrictive. In fact, it implicitly requires that the joint distribution would maintain a Gaussian structure. As we show in the following sections, this assumption makes the problem significantly easier and does not hold in general.

\cite{chalk2016relevant} provide a lower bound to the IB curve by using an approximate variational scheme, analogous to variational expectation maximization. Their method relaxes the IB problem by restricting the class of distributions, $P(\underline{Y}|\underline{T})$ and  $P(\underline{T})$ to a set of parametric models. This way, the relaxed IB problem may be solved in EM-like steps; their suggested algorithm iteratively maximize the objective over the mappings (for fixed parameters) and then maximize the set of parameters, for fixed mappings. \cite{chalk2016relevant} show that this method can be effectively applied to ``sparse" data in which $\underline{X}$ and $\underline{Y}$ are generated by sparsely occurring latent features. However, in the general case, their suggested bound strongly depends on the assumption that the chosen parametric models provide reasonable approximations for the optimal distributions. This assumption is obviously quite restrictive. Moreover, it is usually difficult to validate, as the optimal distributions are unknown.  \cite{kolchinsky2017nonlinear} take a somewhat similar approach, as they suggest a variational upper bound to the IB curve. The main difference between the two methods relies on the variational approximation of objective, $I(\underline{X};\underline{Y})$. However, they are both prune to the same difficulties stated above. 

\cite{alemi2016deep} propose an additional variational inference method to construct a lower bound to the IB curve. Here, they re-parameterize the IB problem followed by Monte Carlo sampling, to get an unbiased estimate of the IB objective gradient. This allows them to apply deep neural networks in order to parameterize any given distribution. However, this method fails to provide guarantees on the obtained bound, as a result of the suggested stochastic gradient decent optimization approach.
 
\cite{achille2016information} relax the bottleneck problem by introducing an additional \textit{total correlation} (TC) regularization term that strives to maximize the independence among the components of the representation $T$. They show that under the assumption that the Lagrange multipliers of the TC and MI constraints are identical, the relaxed problem may be solved by adding auxiliary variables. However, this assumption is usually invalid, and the suggested method fails to provide guarantees on difference between the obtained objective and original IB formulation.    

In this work we suggest a novel lower bound to the IB curve which provides both theoretical and practical guarantees. In addition, we introduce upper and lower bounds for our suggested solution that are very easy to attain. This way we allow immediate benchmarks to the IB curve using common off-shelf methods. 

\section{Problem formulation} \label{Problem Formulation}

Throughout this manuscript we use the following standard notation: underlines denote vector quantities, where their respective components are written without underlines but with index. For example, the components of the $n$-dimensional vector $\underline{X}$ are $X_1, X_2, \dots X_n$. Random variables are denoted with capital letters while their realizations are denoted with the respective lower-case letters. The mutual information of two random variables is defined as $I(\underline{X};\underline{Y})=h(\underline{X})+h(\underline{Y})-h(\underline{X},\underline{X})$ where $h(\underline{X})=-\int_{\underline{X}} f_{\underline{X}}(\underline{x})\log f_{\underline{X}}(\underline{x}) d\underline{x}$ is the differential entropy of $\underline{X}$ and $f_{\underline{X}}(\underline{x})$ is its probability density function.

We begin by introducing a Gaussian lower bound to the mutual information $I(\underline{X};\underline{Y})$. We then extend our result to the entire IB curve.

%$F_{\underline{X}}\left(\b{x}\right)  \triangleq P(X_1 \leq x_1, X_2 \leq x_2\dots) $ is the cumulative distribution function of  $\underline{X}$, while $f_{\underline{X}}\left(\b{x}\right)$ is its probability density function.  Assuming $\underline{X}$ is a continious random variable we define the differential entropy of $\underline{X}$ as h\left(\underline{X}\right)$ is the entropy of $\underline{X}$. This means  $H\left(\underline{X}\right)=-\sum_{\b{x}} P_{\underline{X}}\left(\b{x}\right) \log{P_{\underline{X}}\left(\b{x}\right)}$ where the $\log{}$ function denotes a logarithm of base $2$ and $\lim_{x \to 0} x\log{(x)} = 0$. Further, we denote the binary entropy as $h_b (p)=-p\log{p}-(1-p)\log{(1-p)}$. $I(\underline{X};\underline{Y})$ is the mutual information between 

\subsection{Problem statement}
Let $\underline{X}\in \mathbb{R}^{d_x}, \underline{Y}\in \mathbb{R}^{d_y}$ be two multivariate random vectors with a joint cumulative distribution function (CDF) $F_{XY}(x,y)$ and mutual information $I(\underline{X}, \underline{Y})$.   In the following sections we focus on bounding  $I(\underline{X}, \underline{Y})$ from below with an analytical expression. Let $\underline{U}=\phi(\underline{X})$ and $\underline{V}=\psi(\underline{Y})$ be two transformations of $\underline{X}$ and $\underline{Y}$, respectively. % $\phi: \mathbb{R}^{d_x} \rightarrow \mathbb{R}^{k_x}$ and $\psi: \mathbb{R}^{d_y} \rightarrow \mathbb{R}^{k_y}$   be two transformations such that $k_x\leq d_x$ and $k_y\leq d_y$. Let $\underline{U}=\phi(\underline{X})$ and $\underline{V}=\psi(\underline{Y})$. 
Assume that $\underline{U}$ and $\underline{V}$ are \textit{separately normal distributed}. This means that $\underline{U} \sim N\left(\mu_{\underline{U}},C_{\underline{U}}\right)$ and $\underline{V} \sim N\left(\mu_{\underline{V}},C_{\underline{V}}\right)$ but the vector $[\underline{U}, \underline{V}]^T$ is not necessarily normal distributed. This allows us to derive the following fundamental inequality

\begin{align}
\label{basic_inequality}
I(\underline{X}, \underline{Y}) \geq &I(\underline{U}, \underline{V})=h(\underline{U})+h(\underline{V})-h(\underline{U},\underline{V}) \geq \\\nonumber 
&h(\underline{U})+h(\underline{V})-h(\underline{U}_{jg},\underline{V}_{jg}) = \frac{1}{2} \log \left(\frac{\left|C_{[\underline{U},\underline{V}] }\right|}{|C_{\underline{U}}||C_{\underline{V}}|}   \right)
\end{align}
where the first inequality follows from the Data Processing lemma \citep{cover2012elements} and the second inequality follows from $\left[\underline{U}_{jg},\underline{V}_{jg}\right]^T$ being  jointly Gaussian (jg) distributed with the same covariance matrix as $\left[\underline{U},\underline{V}\right]^T$ , $C_{[\underline{U}_{jg},\underline{V}_{jg}]}=C_{[\underline{U},\underline{V}]}$, so that $h(\underline{U}_{jg},\underline{V}_{jg})\geq h(\underline{U},\underline{V})$ \citep{cover2012elements}. Notice that (\ref{basic_inequality}) can also be derived from an \textit{information geometry} (IG) view point, as shown by  \cite{cardoso2003dependence}.

Equality is attained in ($\ref{basic_inequality}$) iff $I(\underline{X}, \underline{Y}) = I(\underline{U}, \underline{V})$ (no information is lost in the transformation) and $\underline{U}=\phi(\underline{X})$, $\underline{V}=\psi(\underline{Y})$ are jointly normally distributed. In other words, in order to preserve all the information we must find $\phi$ and $\psi$ that  capture all the mutual information, and at the same time make $\underline{X}$ and $\underline{Y}$ jointly normal. This is obviously a complicated task as $\phi$ and $\psi$ only operate on $\underline{X}$ and $\underline{Y}$ separately. 
Therefore, we are interested in maximizing this lower bound as much as possible:
\begin{equation} \label{basic problem}
\begin{aligned}
& {\max_{\phi,\psi}}
& & \log \left(\frac{\left|C_{[\underline{U},\underline{V}] }\right|}{|C_{\underline{U}}||C_{\underline{V}}|}   \right) \\
& \text{subject to}
& & \underline{U}=\phi(\underline{X}) \sim N\left(0,C_{\underline{U}}\right) \\ 
& &&\underline{V}=\psi(\underline{Y}) \sim N\left(0,C_{\underline{V}}\right)
\end{aligned}
\end{equation}
In other words, we would like to maximize \cite{cardoso2003dependence} IG bound by applying two transformations, $\phi$ and $\psi$, to the original variables. This would allow us to achieve a tighter result. 

Notice that our objective is invariant to the means of  $\underline{U},  \underline{V}$ so they are chosen to be zero. In addition, it is easy to show that our objective is invariant to linear scaling of $\underline{U}, \underline{V}$. This means we can equivalently assume that %$\text{cov}(\underline{U}_i, \underline{V}_j)=0$ for $i \neq j$ and that 
$C_{\underline{U}}, C_{\underline{V}}$ are identity covariance matrices.As shown by \cite{kay1992feature} and others \citep{klami2005non,chechik2005information}, maximizing the objective of (\ref{basic problem}) is equivalent to maximizing the canonical correlations, $\text{cov}(\underline{U}_i,\underline{V}_i)$.
Therefore, our problem can be written as

\begin{equation} 
\label{multivariate_problem}
\begin{aligned}
& {\max_{\phi,\psi}}
& & \sum_{i=1}^{k} E\left(\underline{U}_{i} \underline{V} _{i}\right) \\
& \text{subject to}
& & \underline{U}=\phi(\underline{X}) \sim N\left(0,I\right) \\ 
& &&\underline{V}=\psi(\underline{Y}) \sim N\left(0,I\right)\\
\end{aligned}
\end{equation}
where $k=\min\{k_x,k_y\}$. This problem may also be viewed as a variant of the well-known CCA problem \citep{hotelling1936relations}, where we optimize over nonlinear transformations $\phi$ and $\psi$, and impose additional normality constraints. As in CCA, this problem can be solved iteratively by gradually finding the the optimal canonical components in each step (subject to the normality constraint), while maintaining orthogonality with the components that were previously found. For simplicity of the presentation we begin by solving (\ref{multivariate_problem}) in the univariate ($1$-D) case. Then, we generalize to the multivariate case. In each of these setups we present a solution to the problem, followed by simpler upper and lower bounds.  
\section{The univariate case}
\label{The 1-D case}
In the univariate case we assume that $d=k=1$. We would like to find $\phi,\psi$ such that  
\begin{equation} 
\begin{aligned}
\label{1-D problem}
& {\max_{\phi,\psi}}
& & \rho=E(U,V) \\
& \text{subject to}
& & U=\phi(X) \sim N(0,1) \\ 
& && V=\psi(Y)\sim N(0,1)\\
\end{aligned}
\end{equation}
As a first step towards this goal, let us relax our problem by replacing the normality constraint with simpler second order statistics constraints, 

\begin{equation}
\label{ACE problem} 
\begin{aligned}
& {\max_{\phi,\psi}}
& & \rho=E(U,V) \\
& \text{subject to}
& & U=\phi(X),\; E(U)=0, \;E(U^2)=1 \\ 
& && V=\psi(Y), \;E(V)=0, \;E(V^2)=1\\
\end{aligned}
\end{equation}
As mentioned above, this problem is a non-linear extension of CCA, which traces back to early work  by \cite{lancaster1963correlations}. As this problem is also a relaxed version of our original task (\ref{1-D problem}), it may serve us as an upper bound. This means that the optimum of (\ref{ACE problem}), denoted as $\rho_{ub}$,  necessarily bound from above  $\rho_*$, the optimum of (\ref{1-D problem}).

\subsection {Alternation Conditional Expectation (ACE)}
\label{ACE_section}
\cite{breiman1985estimating} show that the optimal solution to (\ref{ACE problem}) is achieved by a simple alternating conditional expectation procedure, named ACE. Assume that $\psi(Y)$ is fixed, known and satistfies the constraints. Then, we optimize (\ref{ACE problem})  only over $\phi$ and by Cauchy-Schwarz inequality, we have that
$$E(\phi(X)\psi(Y))=E_x \left(\phi(X)E(\psi(Y)|X)\right) \leq \sqrt{\text{var}(\phi(X))} \sqrt{\text{var}(E(\psi(Y)|X))} $$
with equality iff $\phi(X)=c\cdot E(\psi(Y)|X)$.  Therefore, choosing the constant $c$ to satisfy the unit variance constraint we achieve $\phi(X)=\frac{E(\psi(Y)|X)}{\sqrt{var(E(\psi(Y)|X))}}$. In the same manner we may fix $\phi(X)$ and attain $\psi(Y)=\frac{E(\phi(X)|Y)}{\sqrt{var(E(\phi(X)|Y))}}$. These coupled equations are in fact necessary conditions for the optimality of $\phi$ and $\psi$, leading to an alternating procedure in which at each step we fix one transformation and optimize the other. \cite{breiman1985estimating} prove that this procedure convergences to the global optimum using Hilbert space algebra. They show that the transformations $\phi$ and $\psi$ may be represented in a zero-mean and finite variance Hilbert space, while the conditional expectation projection is linear, closed, and shown to be self-adjoint and compact under mild assumptions. Then, the coupled equations may be formulate as an eigen problem in the Hilbert space, for which there exists a unique and optimal solution.

The following lemma defines a strict connection between the non-linear canonical correlations and the Gaussinized IB problem. 

\begin{lemma}
\label{negative_lemma}
Let $\rho_{ub}$ be the solution to (\ref{ACE problem}). If $I(X;Y)>-\log\left(1-\rho_{ub}^2\right)$, then there are no transformations $\phi, \psi$ such that $U=\phi(X)$ and $V=\psi(Y)$ are jointly normally distributed and preserve all of the mutual information, $I(X;Y)$.
\end{lemma}

\begin{proof}
Let $\rho_{*}$ be the solution to (\ref{1-D problem}). As mentioned above, $\rho_{ub}\geq \rho_{*}$. Therefore, $I(X;Y)>-\log\left(1-\rho_{ub}^2\right)>-\log\left(1-\rho_{*}^2\right)$. This means that the inequality (\ref{basic_inequality}) cannot be achieved with equality. Hence, there are no transformations  $U=\phi(X)$ and $V=\psi(Y)$ so that $U$ and $V$ are jointly normal and preserve all of the mutual information, $I(X;Y)$.
\end{proof}
Lemma \ref{negative_lemma} suggests that if the optimal transformations of the relaxed problem (which can be obtained by ACE) fails to capture all the mutual information between $X$ and $Y$, then there are no transformations that can project $X$ and $Y$ onto jointly normal variables without losing information. Moreover, notice that the maximal level of correlation $\rho_{ub}$ cannot be further increased, even if we allow $\underline{U}=\phi(X)$ and $\underline{V}=\phi(Y)$ to reside in higher spaces. This means that Lemma \ref{negative_lemma} holds for any $\phi:R\rightarrow R^{k_u}$ and $\psi:R\rightarrow R^{k_v}$, such that $k_u,k_v \geq 0$. 

\subsection{Alternating Gaussinized Conditional Expectations (AGCE)}
\label{AGCE}
Let us go back to our original problem, which strives to maximize the correlation between $U$ and $V$, subject to marginal normality constraints (\ref{1-D problem}). Here we follow \cite{breiman1985estimating}, and suggest an alternating optimization procedure.

 Let us fix $\psi(Y)$ and optimize (\ref{1-D problem}) with respect to $\phi(X)$. As before, we can write the correlation objective as  $E(\phi(X)\psi(Y))=E_x \left(\phi(X)E(\psi(Y)|X)\right)$. Since $E(\phi(X)^2)$ is constrained to be equal to $1$, while $E\left(E(\psi(Y)|X)^2\right)$ is fixed, maximizing $E_x \left(\phi(X)E(\psi(Y)|X)\right)$ is equivalent to minimizing $E_x \left(\phi(X)-E(\psi(Y)|X)\right)^2$. For simplicity, denote $\bar{X} \equiv E(\psi(Y)|X)$. Then, our optimization problem can be reformulated as 

\begin{equation}
\label{AGCE problem} 
\begin{aligned}
& {\min_{\phi}}
& & E \left(\phi(\bar{X})-\bar{X}\right)^2 \\
& \text{subject to}
& & \bar{X} \sim F_{\bar{X}}\\
& && \phi(\bar{X}) \sim N(0,1) \\ 
\end{aligned}
\end{equation}
where $F_{\bar{X}}$ is the (fixed) CDF of $\bar{X} \equiv E(\psi(Y)|X)$. Notice that $\phi$ is necessarily a function of $\bar{X}$ alone (as opposed to $X$), for simple optimization considerations. %as $E \left( \left(\phi(X)-\bar{X}\right)^2 \right)=E \left( E \left( \left(\phi(X)-\bar{X}\right)^2 | \bar{X} \right) \right) = E \left(  E (\phi(X) | \bar{X}) -\bar{X})^2 \right)$, and denote $E (\phi(X) | \bar{X}) \equiv \bar{\phi}(\bar{X})$.
Assuming that $\bar{X}$ and $U=\phi(\bar{X})$ are two separable metric spaces such that any probability measure on $\bar{X}$ (or $U$) is a Radon measure (i.e. they are Radon spaces), then (\ref{AGCE problem}) is simply an optimal transportation problem \citep{monge1781memoire} with a strictly convex cost function (mean square error). We refer to $\phi^*(\bar{X})$ that minimizes (\ref{AGCE problem}) as the optimal map.

The optimal transportation problem was presented by \cite{monge1781memoire} and has generated an important branch of mathematics. The problem originally studied by Monge was the following: assume we are given a pile of sand (in $\mathbb{R}^3$) and a hole that we have to completely fill up with that sand. Clearly the pile and the hole must have the same volume and different ways of moving the sand will give different costs of the operation. Monge wanted to minimize the cost of this operation. Formally, the optimal transportation problem is defined as 
$$\inf \left\{\int_{\bar{X}}c(\bar{X},\phi(\bar{X}))d\mu(\bar{X})  \Big| \phi_*(\mu)=\nu \right\}$$
where $\mu$ and $\nu$ are the probability measures of $\bar{X}$ and $U$ respectively, $c(\cdot,\cdot)$ is some cost function and $\phi_*(\mu)$ denotes the push forward of $\mu$ by the map $\phi$.
Clearly, (\ref{AGCE problem}) is a special case of the optimal transportation problem  where the  $\mu=F_{\bar{X}}$, $\nu$ is a standard normal distribution and the cost function is the euclidean distance between the two. 

Assume that $\bar{X} \in \mathbb{R}$ has finite $p^{th}$ moments for $1 \leq p < \infty$ and a strictly continuous CDF, $F_{\bar{X}}$ (that is $\bar{X}$ is a strictly continuous random variable). Then, \cite{rachev1998mass} show that the optimal map (which minimizes (\ref{AGCE problem})) is exactly $\phi^*(\bar{X})=\Phi^{-1}_N \circ  F_{\bar{X}} (\bar{X})$ where $\Phi^{-1}_N$ is the inverse CDF of a standard normal distribution. As shown by \cite{rachev1998mass}, the optimal map is unique and achieves 
\begin{equation}
\label{transportation_problem_loss} 
E \left(\left(\phi^*(\bar{X})-\bar{X}\right)^2\right) =\int_0^1 \left(F_{\bar{X}}(s)-  \Phi_N(s)  \right)^2 ds. 
\end{equation}  Notice that the optimal map may be generalized to the multivariate case, as discussed in the next Section. The solution to the optimal transportation problem is in fact the ``optimal projection" of our problem (\ref{AGCE problem}). Further, it allows us to quantify how much we lose from imposing the marginal normality constraint, compared with ACE's optimal projection.

Notice that the optimal map, $\phi^*(\bar{X})=\Phi^{-1}_N \circ  F_{\bar{X}} (\bar{X})$, is simply marginal Gaussianization of $\bar{X}$: applying $\bar{X}$'s CDF to itself results in a uniformly distributed random variable, while $\Phi^{-1}_N$ shapes this uniform distribution into a standard normal. In other words, while the optimal projection of $\psi(Y)$ on $X$ is its conditional expectation, the optimal projection under a normality constraint is simply a Gaussianization of the conditional expectation. The uniqueness of the optimal map leads to the following necessary conditions for an optimal solution to (\ref{1-D problem}),

\begin{align}
\label{necessary_conditions}
&\phi(X)=\Phi^{-1}_N \circ  F_{E(\psi(Y)|X)} (E(\psi(Y)|X))\\\nonumber
&\psi(Y)=\Phi^{-1}_N \circ  F_{E(\phi(X)|Y)} (E(\phi(X)|Y))
\end{align}
  
As in ACE, these necessary conditions imply an alternating projection algorithm, namely, the Alternating Gaussinized Conditional Expectation (AGCE). Here, we begin by randomly choosing a transformation that only satisfies the normality constraint $\psi(Y) \sim N(0,1)$. Then, we iterate by fixing one of the transformation while optimizing the other, according to (\ref{necessary_conditions}). We terminate once $E(\phi(X)\psi(Y))$ fails to increase, which means that we converged to a set of transformations that satisfy the necessary conditions for optimal solution.  Notice that in every step of our procedure, we may either:

\begin{enumerate}
\item	Increase our objective value, as a result of the optimal map for (\ref{AGCE problem}).
\item	Maintain with the same objective value and with the same transformation that was found in of the previous iteration, as we converged to  (\ref{necessary_conditions}).  
\end{enumerate}
This means that our alternating method generates a monotonically increasing sequence of objective values. Moreover, as shown in Section \ref{The 1-D case}, this sequence is bounded from above by the optimal correlation given by ACE. Therefore, according to the monotone convergence theorem, our suggested method converges to a local optimum. 

Unfortunately, as opposed to ACE, our projection operator is not linear and we cannot claim for global optimality. We see that for different random initializations we converge to (a limited number) of local optima. Yet, AGCE provides an effective tool for finding local maximizers of (\ref{The 1-D case}), which together with MCMC \citep{gilks2005markov} initializations (or any other random search mechanisms) is capable of finding the global optimum.

\subsection{Off-shelf lower bound}
\label{offshelf_lower_bound}
Although the AGCE method provides a (locally) optimal solution to (\ref{The 1-D case}), we would still like to consider a simpler ``off-shelf" mechanism that is easier to implement and gives a lower bound to the best we can hope for. Here, we tackle (\ref{The 1-D case}) in two phases. In the first phase we would like to maximize the correlation objective, $E(UV)$, subject to the relaxed second order statistics constraints (as defined in (\ref{ACE problem})). Then, we enforce the marginal normality constraints by simply applying \textit{separate Gaussianization} to the outcome of the first phase. In other words, we first apply ACE to increase our objective as much as possible, and then separately Gaussianize the results to meet the normality constraints, hoping this process does not reduce our objective ``too much". Notice that in this univariate case, separate Gaussianization is achieved according to Theorem \ref{gaussianization_theorem}:
\begin{theorem}
\label{gaussianization_theorem}
Let $X$ be any random variable $X \sim F_X (x)$  and $\theta \sim \text{Unif}[0,1]$ be statistically independent of it. In order to shape $X$ to a normal distribution the following applies:
\begin{enumerate}
\item	Assume $X$ is a non-atomic distribution ($F_X (x)$ is strictly increasing) then  $\Phi^{-1}_N \circ F_X(X)\sim N(0,1)$
\item	Assume $X$ is discrete or a mixture probability distribution then  $\Phi^{-1}_N \circ \left( F_X(X)-\theta P_X(x)\right) \sim N(0,1)$
\end{enumerate}
\end{theorem}
\noindent The proof of this theorem can be located in Appendix $1$ of \citep{shayevitz2011optimal}. Theorem \ref{gaussianization_theorem} implies that if $X$ is strictly continuous then we may achieve a normal distribution by applying $\Phi^{-1}_N \circ F_X(X)$ to it, as discussed in the previous section. Otherwise, we shall handle its CDF's singularity points by  randomly scattering them in a uniform manner, followed by applying  $\Phi^{-1}_N$ to the random variable we achieved. Notice that this process do not allow any flexibility in the Gaussianization process. However, we show that in the multivariate case (Section (\ref{multivariate LB})) the equivalent process is quite flexible and allows us to control the correlation objective.

Further, notice that this lower bound is by no means a candidate for an optimal solution to (\ref{1-D problem}), as it does not meet the necessary conditions described in (\ref{necessary_conditions}). Yet, by finding both an upper and lower bounds (through ACE, and then separately Gaussianizing the result of ACE) we may immediately achieve the range in which the optimal solution necessarily resides. Assuming this range is not too large, one may settle for a sub-optimal solution without a need to apply AGCE at all.  

\subsection{Illustrative example}
\label{examples_1}
We now demonstrate our suggested methodology with a simple illustrative example. 
Let $X \sim N(0,1)$, $W \sim N(0,\epsilon^2)$ and $Z\sim N(\mu_z,1)$ be three normally distributed random variables, all independent of each other.
Let $P$ be a Bernoulli distributed random variable with a parameter $\frac{1}{2}$, independent of $X,W$ and $Z$. Define $Y$ as:
\begin{align}
Y=\left\{\begin{tabular}{ l c  }
  X+W & P=0  \\
  Z & P=1  \\
 \end{tabular}\right\}.\nonumber
\end{align}
Then, $Y$ is a balanced Gaussian mixture with parameters 
$$\theta_y=\left\{\mu_1=0, \sigma_1^2=1+\epsilon^2, \mu_2=\mu_z,  \sigma_2^2=1\right\}.$$
The joint probability density function of $X$ and $Y$ is also a balanced two-dimensional Gaussian mixture with parameters 

$$\theta_{xy}=\left\{\mu_1=\left[ \begin{tabular}{ l c  }
  0 \\
  0  \\
 \end{tabular}\right] ,C_1=\left[\begin{tabular}{ l c  }
  1 & 1  \\
  1 & 1+$\epsilon^2$  \\
 \end{tabular}\right], \mu_2=\left[ \begin{tabular}{ l c  }
  0 \\
  $\mu_z$  \\
 \end{tabular}\right] ,C_2=I\right\}. $$

\noindent Let us further assume that $\mu_z$ is large enough, and $\epsilon^2$ is small enough, so that the overlap between the two Gaussian is negligible. For example, we set $\mu_z=10$ and $\epsilon=0.1$. The correlation between $X$ and $Y$ is easily shown to be $\rho_{xy}=\frac{\sfrac{1}{2}}{\sqrt{1+\sfrac{1}{2}\epsilon^2+\sfrac{1}{4}\mu_z}}=0.098$. The mutual information between $X$ and $Y$ is defined as
$$I(X;Y)=h(X)+h(Y)-h(X;Y)$$  
Since we assume that the Gaussians in the mixture practically do not overlap, we have that
\begin{align}
h(Y)=&-\int f_Y(y)\log f_Y(y)dy \approx%\\\nonumber
%& -\frac{1}{2}\int N\left(0,1+\epsilon^2\right)\log \left( \frac{1}{2}N\left(\mu_z,1+\epsilon^2\right)\right)-\frac{1}{2}\int N(\mu_z,1)\log \left( \frac{1}{2}N(\mu_z,1) \right)=\\\nonumber
%&\frac{1}{2}h\left( N\left(0,1+\epsilon^2\right) \right)+\frac{1}{2}h\left( N\left(\mu_z,1\right) \right)+1=\\\nonumber&
\frac{1}{4}\log\left(2\pi e (1+\epsilon^2)\right)+\frac{1}{4}\log\left(2\pi e \right)+1
\end{align}
%where $N(\mu,\sigma^2)$ denotes the probability density function of a normal distribution with mean $\mu$ and variance $\sigma^2$.
In the same manner, %we have that 
\begin{align}
h(X,Y)=&-\int f_{X,Y}(x,y)\log f_{X,Y}(x,y)dxdy \approx\\\nonumber
%& -\frac{1}{2}\int N\left(\mu_1,C_1\right)\log \left( \frac{1}{2}N\left(\mu_1,C_1\right)\right)-\frac{1}{2}\int N\left(\mu_2,C_2\right)\log \left( \frac{1}{2}N\left(\mu_2,C_2\right) \right)=\\\nonumber
&\frac{1}{4}\log\left((2\pi e)^2 |C_1|\right)+\frac{1}{4}\log\left((2\pi e)^2 |C_2| \right)+1
\end{align}
Plugging $\mu_z=10$ and $\epsilon=0.1$ we have that
\begin{align}
I(X;Y)=&h(X)+h(Y)-h(X;Y) \approx %\\\nonumber
%&\frac{1}{2}\log\left(2\pi e \right)+ \frac{1}{4}\log\left(2\pi e (1+\epsilon^2)\right)+\frac{1}{4}\log\left(2\pi e \right)+1-\\\nonumber
%&\frac{1}{4}\log\left((2\pi e)^2 |C_1|\right)+\frac{1}{4}\log\left((2\pi e)^2 |C_2| \right)+1=
1.66 \text{bits}.
\end{align} The scatter plot on the left of Figure \ref{one_d_example_1} illustrates $10,000$ independent draws of $X$ and $Y$, where the blue circles corresponds to the ``correlated samples" ($P=0$) while the blue crosses are the ``noise" ($P=1$). 

Before we proceed to apply our suggested methods, let us first examine two benchmark options for separate Gaussianization. As an immediate option, we may always apply separate Gaussianization, directly to $X$ and $Y$, denoted as $U_a$ and $V_a$ respectively. This corresponds to \cite{cardoso2003dependence} information geometry bound. Since $X$ is already normally distributed we may set $U_a=X$ and only apply Gaussinization to $Y$. Let $V_a=\psi(Y)$ be the Gaussianization of $Y$. This means that $$V_a=\Phi_N^{-1}\left(F_{Y}\left(Y\right)\right)= \Phi_N^{-1}\left( \Phi_{GM(\theta_y)}(Y)\right)$$

\noindent where $\Phi_{GM(\theta_y)}$ is the cumulative distribution function a Gaussian Mixture with the parameters $\theta_y$ described above. Therefore,
$$\rho_{u_a,v_a}=E(XV)=\frac{1}{2}E\left(X \Phi_N^{-1}\left( \Phi_{GM(\theta_y)}(X+W)\right)\right).$$
Although it is not possible to obtain a closed form solution to this expectation,  it may be numerically evaluated quite easily, as $X$ and $W$ are independent. Assuming $\mu_z=10$ and $\epsilon=0.1$ we get that $\rho_{u_a,v_a} \approx 0.288$ and our lower bound on the mutual information, as appears in (\ref{basic_inequality}), is 
$I_g \equiv -\frac{1}{2}\log\left(1-\rho_{u_a,v_a}^2\right)\approx0.0628\text{bits}$. The middle scatter plot of Figure \ref{one_d_example_1} presents this separate marginal Gaussianization of the previously drawn $10,000$ samples of $X$ and $Y$. Notice that the marginal Gaussianization is a monotonic transformation, so that the $Y$ samples are not being shuffled and maintain the separation between the two parts of the mixture. 
While the red circles are now ``half Gaussian", the blue crosses are shaped in a curvy manner, so that their marginal distribution (projected on the $y$ axis) is also a ``half Gaussian", leading to a normal marginal distribution of $Y$.  We notice that while the mutual information between $X$ and $Y$ is $1.66$ bits, the lower bound attained by this naive Gaussianization approach is close to zero. This is obviously an unsatisfactory result.

A second benchmark alternative for separate Gaussianization is to take advantage of the Gaussian mixture properties. Since we assume that the two Gaussians of $Y$ are practically separable, we may distinguish between observations from the two Gaussians. Therefore, we can simply reduce $\mu_z$ from the $Z$ samples (the red circles), and normalize the observations of $X+W$. This way the transformed $Y$ becomes a Gaussian mixture of two co-centered standard Gaussians, and no further Gaussianization is necessary. For $\mu_z=10$ and $\epsilon=0.1$, this leads to a correlation of 
\begin{align}
\rho_{u_b,v_b}=\frac{1}{2}E\left( \frac{1}{\sqrt{1+\epsilon^2}}(X+W)X\right)=\frac{1}{2} \frac{1}{\sqrt{1+\epsilon^2}}=0.497
\end{align}
and a corresponding mutual information lower bound of $I_g=0.204 \,\text{bits}$. However, notice that the suggested transformation is not invertible and may cause a reduction in mutual information.
Specifically, we now have that the joint distribution of $U_b=X$ and $V_b$ follows a Gaussian mixture model with parameters:

$$\theta_{u_b, v_b}=\left\{\mu_1=\left[ \begin{tabular}{ l c  }
  0 \\
  0  \\
 \end{tabular}\right] ,C_1=\left[\begin{tabular}{ l c  }
  1 & $\frac{1}{\sqrt{1+\epsilon^2}}$ \\
  $\frac{1}{\sqrt{1+\epsilon^2}}$ & 1  \\
 \end{tabular}\right], \mu_2=\left[ \begin{tabular}{ l c  }
  0 \\
  0  \\
 \end{tabular}\right] ,C_2=I\right\} $$
Therefore,
\begin{align}
h(U_b,V_b)=&-\int f_{U_b,V_b}(u,v)\log \left( f_{U_b,V_b}(u,v)\right)dudv=\\\nonumber
&-\int \phi_{GN(\theta_{u_b, v_b})}(u,v)\log \phi_{GN(\theta_{u_b,v_b})}(u,v)dudv \approx 3.1384 \text{bits}
\end{align}
where $\phi_{GN(\theta_{u_b,v_b})}(u,v)$ is the probability density function of a Gaussian mixture with the parameters $\theta_{u_b, v_b}$ described above, and the last approximation step is due to numerical integration. This leads to $I(U_b;V_b)=0.95$ bits.

To conclude, although the mutual information is reduced from $1.66$ bits to $0.95$ bits, the suggested  bound increased quite dramatically, from $0.0628$ bits to $0.204$ bits. The right plot of Figure \ref{one_d_example_1} demonstrates this customized separate Gaussianization (as it only applies for this specific setup) to the previously sampled $X$ and $Y$. Again, we emphasis that this solution is not applicable in general, and is only feasible due to the specific nature of this Gaussian mixture model.

\begin{figure}[ht]
\centering
\includegraphics[width =\textwidth,bb= 90 110 710 500,clip]{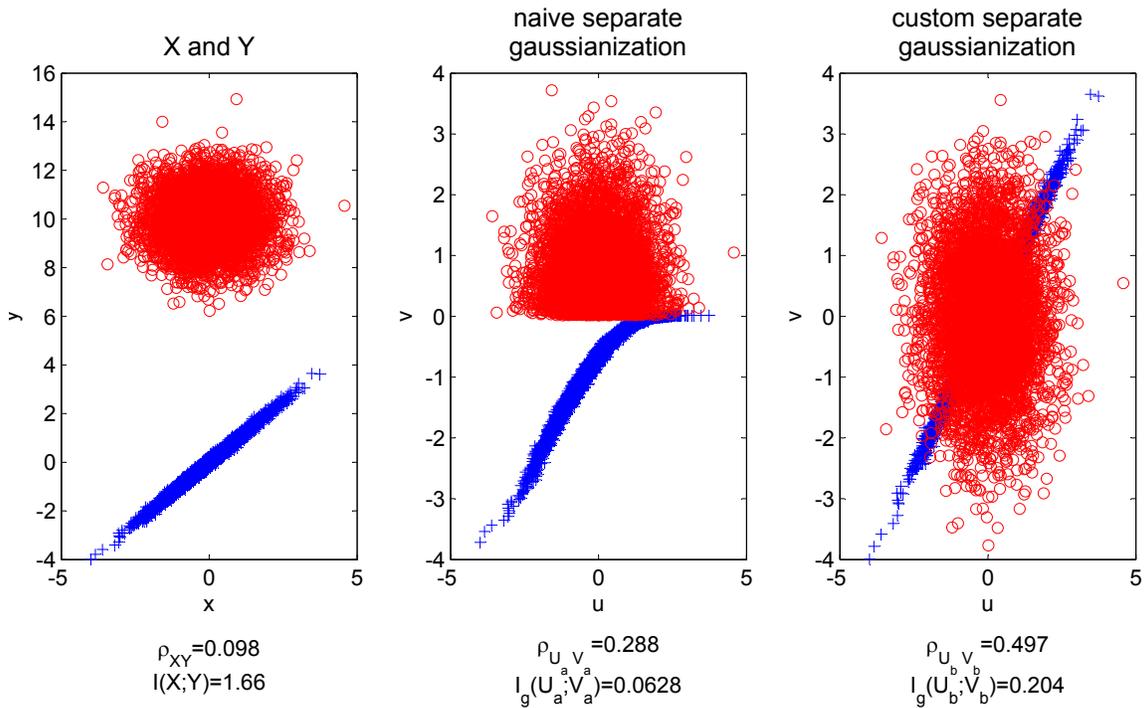}
\caption{Naive Univariate Gaussianization: Left: scatter of $X$ and $Y$, as described in the text. Middle: naive separate Gaussianization to $X$ and $Y$. Right: separate Gaussianization which considers the separable Gaussian Mixture model of  $X$ and $Y$, as described in the text.  }
\label{one_d_example_1}
%\vspace{-1.5em}
\end{figure}

Let us now turn to our suggested methods, as described in detail in the previous sections. 
We begin by applying the ACE procedure (Section \ref{ACE_section}), to attain an upper bound on our problem (\ref{1-D problem}). Not surprisingly, ACE converges to a solution in which the samples of  $Y$ that are independent of $X$ (the ones that come from $Z$) are set to zero, while the rest are normalized to achieve an unit variance. Therefore, the resulting correlation is $\rho_{ub}=\frac{\sfrac{1}{2}}{\sqrt{\sfrac{1}{2}(1+\epsilon^2)}}=0.703$. This results further implies that we can never find a Gaussianization procedure that will capture all the information between $X$ and $Y$, as $I(X;Y)>-\log\left(1-\rho_{ub}^2\right)=0.4917$ bits, according to Lemma \ref{negative_lemma}.  The left scatter plot of Figure \ref{one_d_example_2} demonstrates the outcome of the ACE procedure, applied to the drawn $10,000$ samples of $X$ and $Y$.

Next, we apply our suggested AGCE routine, described in Section \ref{AGCE}. As discussed above, the AGCE only converges to a local optimum. Therefore, we initialize it with several random transformation (including the ACE solution that we just found). We notice that the number of convergence points are very limited and result in almost similar maxima. The middle scatter plot of Figure \ref{AGCE} shows the best result we achieve, leading to a correlation coefficient of $0.66$ and a lower bound on a corresponding Gaussian lower bound (\ref{basic_inequality}) of $0.411$ bits. This result demonstrates the power of our suggested approach, as it significantly improves the benchmarks, even compared with the $U_b, V_b$ that considers the separable Gaussian mixture nature of our samples. 

Finally, we evaluate a lower bound for (\ref{1-D problem}), as described in Section \ref{offshelf_lower_bound}. Here, we simply apply separate Gaussianization to the outcome of the ACE procedure. This results in $\rho_{lb}=0.646$ and a corresponding $I_g=0.389$. The right scatter plot of Figure \ref{one_d_example_2} shows the Gaussianized samples the we achieve. We notice that this lower bound is not significantly lower than AGCE, suggesting that in some case we may settle for this less involved method.

To conclude, our suggested solution surpasses the benchmarks quite easily, as we increase the lower bound from $0.204$ bits using the custom Gaussianization procedure to $0.411$ bits using our general solution. We notice that all of the discussed procedures result in a joint distribution that are quite far from normal. This is not surprising, since $X$ and $Y$ were highly ``non-normal" to begin with. Specifically, in all suggested procedures we loose information, compared with the original $I(X;Y)=1.66$. However, our suggested solution minimizes this loss, and may be considered ``more jointly normal" than others, in this regards.

\begin{figure}[ht]
\centering
\includegraphics[width =\textwidth,bb= 90 85 710 525,clip]{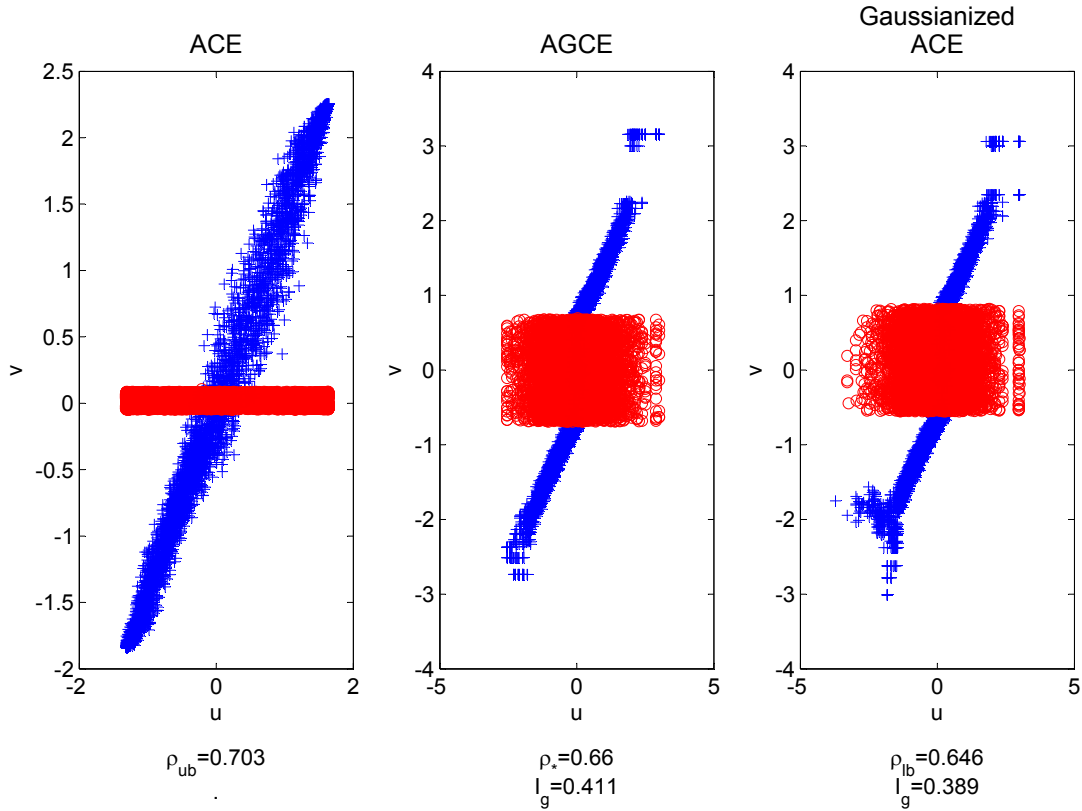}
\caption{Our Suggested Univariate Gaussianization Schemes: Left: upper bound by ACE. Middle: (local) optimal solution by AGCE. Right: lower bound by separate Guassianization to ACE.  }
\label{one_d_example_2}
%\vspace{-1.5em}
\end{figure}

\section{The multivariate case} \label{multivariate case}

Let us now consider the multivariate case where both $\underline{X} \in \mathbb{R}^{d_x}$ and $\underline{Y} \in \mathbb{R}^{d_y}$ are random vectors with a joint CDF $F_{\underline{X},\underline{Y}}$. One of the fundamental differences from the univariate case is that Gaussianizing each of these vectors (even separately) is not a simple task. In other words, finding a transformation $\phi:  \mathbb{R}^{d_x} \rightarrow \mathbb{R}^{k_x}$ such that $\underline{U}=\phi(\underline{X})$ is normally distributed may be theoretically straight-forward but practically involved.

For simplicity of the presentation, assume that $\underline{X}=[X_1, X_2]^T$ is a two dimensional, strictly continuous, random vector. Then, Gaussianization may be achieved in two steps: first, apply marginal Gaussianization to $X_1$, so that $U_1=\Phi^{-1}_N \circ  F_{X_1} (X_1)$. Then, apply marginal Gaussianization on $X_2$, conditioned on each possible realization of the previous component,  $U_2|u_1 = \Phi^{-1}_N \circ  F_{X_2|U_1} (X_2|U_1=u_1)$. This results in a jointly normally distributed vector $\underline{U}=[U_1, U_2]^T$. While this procedure is theoretically simple, it is quite problematic to apply in practice, as it requires Gaussianizing each and every conditional CDF. This is obviously impossible, given a finite number of samples. Yet, it gives us a constructive method, assuming that all the CDF's are known. In the following sections we shall present several alternatives for Gaussianization in finite sample size setup.   

\subsection{Upper bound by ACE} 
\label{multivariate_ub}
As in the univariate case, we begin our analysis by relaxing the normality constraints with softer second order statistics constraints. This leads to a straight forward multivariate generalization of the ACE procedure:

We begin by extracting the first canonical pair, which satisfies $U_1=c \cdot E(V_1|\underline{X})$ and $V_1=c \cdot E(U_1|\underline{Y})$. As in the univariate case, $c$ is a normalization coefficient (the square root of the variance of the conditional expectation), and the optimization is done by alternating projections.
Then, we shall extract the second pair of canonical components, subject to an orthogonality constraint with the first pair. It is easy to show that if $V_2$ is orthogonal to $V_1$, then $U_2= c\cdot E(V_2|\underline{X})$ is orthogonal to $U_1$, and obviously maximizes the correlation with $V_2$. Therefore, we may extract the second canonical pair by first randomly assigning a zero-mean and unit variance $V_2$ that is also orthogonal to $V_1$ (by Gram-Schmidt procedure, for example), followed by alternating conditional expectations with respect to $V_2$ and $U_2$, in the same manner as we did with the first pair. We continue this way for the rest of the canonical pairs. As in the univariate case, convergence to a global maximum is guaranteed from the same Hilbert space arguments. As before, the multivariate ACE sets an upper bound to (\ref{multivariate_problem}) as it maximizes a relaxed version of this problem. 

\begin{lemma}
\label{negative_lemma_2}
Let $\underline{U}_*,\underline{V}_*$ be the outcome of multivariate ACE procedure (the canonical vectors). Assuming that $I(\underline{X};\underline{Y})>\log \left|C_{[\underline{U}_*,\underline{V}_*] }\right|$, there are no transformations such that $\underline{U}=\phi(\underline{X})$ and $\underline{V}=\psi(\underline{Y})$ follow a jointly normal distribution and preserve all of the mutual information, $I(X;Y)$.
\end{lemma}
The proof of Lemma \ref{negative_lemma_2} follows exactly the proof of Lemma \ref{negative_lemma}. Here again, the multivariate ACE objective, $\log \left|C_{[\underline{U}_*,\underline{V}_*] }\right|$, cannot be further increased by artificially inflating the dimension of the problem. Therefore, Lemma \ref{negative_lemma_2} holds for any $\phi:\mathbb{R}^{d_x}\rightarrow \mathbb{R}^{k_x}$ and $\psi:\mathbb{R}^{d_y}\rightarrow \mathbb{R}^{k_y}$, such that $k_x,k_y \geq 0$.  

\subsection{multivariate AGCE}
\label{multivariate AGCE}
As with the multivariate ACE, we propose a generalized multivariate procedure for AGCE. We begin by extracting the first pair, in the same manner as we did in the univariate case. That is, we find a pair $U_1$ and $V_1$ that satisfies 
\begin{align}
\label{necessary_conditions_2}
&U_1=\Phi^{-1}_N \circ  F_{E(U_1|\underline{X})} (E(U_1|\underline{X}))\\\nonumber
&V_1=\Phi^{-1}_N \circ  F_{E(V_1|\underline{Y})} (E(V_1|\underline{Y}))
\end{align} 
by applying the alternating optimization scheme. As we proceed to the second pair, we require that $U_2$ is both orthogonal and jointly normally distributed with $U_1$ (same goes for $V_2$ with respect to $V_1$). This means that the second pair needs not only to be orthogonal, but also statistically independent with the first pair. In other words, assuming $V_2$ is fixed, our basic projection step is

 \begin{equation}
\label{AGCE problem 2} 
\begin{aligned}
& {\max_{\phi_2}}
& & E \left(\phi_2(\underline{X}) V_2\right) \\
& \text{subject to}
& & \phi_2(\underline{X}) \sim N(0,1)\\
& && \phi_2(\underline{X}) \indep \phi_1(\underline{X}) \\ 
\end{aligned}
\end{equation}
Let us denote a subspace $\tilde{\underline{X}} \subset \underline{X}$ that is statistically independent of $U_1=\phi_1(\underline{X})$. Then, the problem of maximizing $E\left(\phi_2(\tilde{\underline{X}}) V_2\right)$ subject to $\phi_2(\tilde{\underline{X}}) \sim N(0,1)$ is again solved by the optimal map, $\phi_2(\tilde{\underline{X}})=\Phi^{-1}_N \circ  F_{E(V_2|\tilde{\underline{X}})} (E(V_2|\tilde{\underline{X}}))$. Therefore, the remaining task is to find the ``best" subspace $\tilde{\underline{X}} \subset \underline{X}$, so that $E\left(\phi_2(\tilde{\underline{X}}) V_2\right)$ is maximal, when plugging the optimal map.

\begin{proposition}
Let $U_1=u_1$ be the value (realization) of $U_1$. Let $\tilde{\underline{X}} = g \left(\underline{X}, u_1\right)$ be a subspace of $\underline{X}$, independent of $U_1$. If $g \left(\underline{X}, u_1\right)$ is an invertible function with respect to $\underline{X}$ given $u_1$, then $\tilde{\underline{X}}$ is an optimal subspace for maximizing $E\left(\phi_2(\tilde{\underline{X}}) V_2\right)$ subject to $\phi_2(\tilde{\underline{X}}) \sim N(0,1)$.
\end{proposition}

\begin{proof}
Assume there exists a different subspace $\tilde{\underline{X}}'=g' \left(\underline{X} , u_1\right)$ so that  $$\max_{\phi_2'} E\left(\phi_2'(\tilde{\underline{X}}') V_2\right)>\max_{\phi_2} E\left(\phi_2(\tilde{\underline{X}}) V_2\right)$$ subject to the normality constraint.
Since $g$ is invertible we have that $\underline{X}=g^{-1}( \tilde{\underline{X}}, u_1)$. Therefore, $\tilde{\underline{X}}'= g'\left(g^{-1}( \tilde{\underline{X}},u_1)\right) \equiv f(\tilde{\underline{X}},u_1)$. Plugging  this to the inequality above leads to  
$$\max_{\phi'_2} E\left(\phi'_2( f(\tilde{\underline{X}},u_1)) V_2\right)>\max_{\phi_2} E\left(\phi_2(\tilde{\underline{X}}) V_2\right)$$
which obviously contradicts the optimality of maximization over $\phi_2$.
\end{proof}
Therefore, we are left with finding $\tilde{\underline{X}} = g \left(\underline{X}, u_1\right)$ that is a subspace of $\underline{X}$, independent of $U_1$ and invertible with respect to $\underline{X}$ given $u_1$. For simplicity of th presentation, let us first assume that $X$ is univariate. Then, the function $g \left(X, u_1\right)= F_{X|U_1}(X|U_1=u1)$ is independent of $U_1$ (as it holds the same  (uniform) distribution, regardless to the value of $U_1$), and invertible given $u_1$ (assuming that the conditional CDF's $F_{X|U_1}(X|U_1=u1)$ are continuous for every $u_1$).
Going back to the multivariate $\underline{X}\in \mathbb{R}^{d_x}$, we may follow the same rationale by choosing a single ${d_x}$-dimensional distribution that all the conditional CDF's, $F_{\underline{X} | U_1}$ will be transformed to. For simplicity we choose a ${d_x}$-dimensional uniform distribution, denoted by its CDF as $F_{unif}$. Then, $g_* \left(F_{\underline{X}|U_1}, u_1\right) = F_{unif}$, where $g_* (P, x)=Q$ refers to a mapping that pushes forward the distribution $P$ into $Q$, given x. Specifically, if $p(w)$ and $q(w)$ are the corresponding density functions of the (absolutely continuous) CDF's $P$ and $Q$ respectively, then we know from basic probability theory that the push forward transformation $S$ satisfies
$$ p(w)=q\left(S(w)\right) |J_S\left(S(w)\right)|$$
where $J_S$ is the Jacoby operator of the map $S$.

To conclude, in order to construct $\tilde{\underline{X}}$ that is independent of $U_1$ and invertible given $u_1$, we need to push forward all the conditional CDF's $F_{\underline{X}|U_1}(\underline{X}|U_1=u_1)$ into a predefined distribution (say, uniform). Then, the optimal map $\phi_2(\tilde{\underline{X}})$ that maximizes $E\left(\phi_2(\tilde{\underline{X}}) V_2\right)$ subject to $\phi_2(\tilde{\underline{X}}) \sim N(0,1)$ is given by $\phi_2(\tilde{\underline{X}})=\Phi^{-1}_N \circ  F_{E(V_2|\tilde{\underline{X}})} (E(V_2|\tilde{\underline{X}}))$. In the same manner, we may find $\tilde{\underline{Y}}$ that is independent of $V_1$ and invertible given $v_1$, and carry on with the alternating projections. This process continues for all the Gaussinized canonical components and converges to a local optimum, from the same considerations described in the univariate case. 

It is important to notice that while this procedure may be considered practically infeasible (as it requires estimating the conditional CDF's), it is equivalently impractical as the multivariate Gaussianization considered in the beginning of this section. Yet, it gives us a local optimum for our problem, assuming that we know the joint probability distribution.

\subsection{Off-shelf lower bound in the multivariate case}
\label{multivariate LB}
In the same manner as with the univariate case, we may apply a simple off-shelf lower bound to (\ref{basic problem}) by first maximizing the objective as much as we can (using multivariate ACE) followed by Gaussianizing the outcome vectors, hoping we do not reduce the objective ``too much". However, as mentioned in the beginning of Section \ref{multivariate case}, applying multivariate Gaussianization may be practically infeasible. Therefore, we begin this section by reviewing practical multivariate Gaussianization methodologies. Then, we use these ideas to suggest a practical lower bound, which unlike the univariate case, is not oblivious to our objective. 

\subsubsection{Practical multivariate Gaussianization}  
\label{Practical multivariate Gaussianization}
The Gaussianization procedure strives to find a transformation $\underline{Z}=\mathcal{G}(\underline{X})$ so that $\underline{Z} \sim N(0,I)$. A reasonable a cost function for describing ``how Gaussian" $\underline{Z}$ really is, may be the Kullback Leibler Divergence (KLD) between $\underline{Z}$'s PDF, $f_{\underline{Z}}(\underline{z})$, and a standard normal distribution,

$$ J(\underline{Z})=D_{KL} \left(f_{\underline{Z}}(\underline{z}) || f_N(\underline{Z}) \right)=\int_{\underline{Z}} f_{\underline{Z}}(\underline{z}) \log \left( \frac{f_{\underline{Z}}(\underline{z})}{f_N(\underline{Z})}   \right)dz$$  
where $f_N(\underline{Z})$ is the PDF of a standard normal distribution. As shown by \cite{chen2001gaussianization}, $J(\underline{Z})$ may be decomposed into

\begin{equation}
\label{KLD}
J(\underline{Z})=D_{KL} \left(f_{\underline{Z}}(\underline{z}) ||\prod_{i=1}^{d_z} f_{Z_i}(z_i) \right)+\sum_{i=1}^{d_z} D_{KL} \left(f_{Z_i}(z_i) ||f_N(z)\right)
\end{equation}
where the first KLD term measures how independent are the components of $\underline{Z}$, while the second term indicates how normally distributed is each component. This decomposition led \cite{chen2001gaussianization} to an iterative algorithm. In each iteration, their suggested approach applies Independent Component Analysis \citep{hyvarinen2004independent}, to minimize the first term, followed by marginal Gaussianization of each component (as we describe for the univariate case), to minimize the second term. Chen and Gopinath show that minimizing one term does not effect the other, which leads to a monotonically decreasing procedure that converges once $\underline{Z}$ is normally distributed.

Notice that the Independent Component Analysis (ICA) is a linear operator. Therefore, if $\underline{Z}$ can be linearly decomposed into independent components, then Chen and Gopinath's Gaussianizion process converges in a single step. Moreover, notice that this Gaussianization process does not require estimating the multivariate distribution. However, it does require estimating the marginals, $f_{Z_i}$ which is, in general, considered a much easier task. %Further discussion regarding marginal Gaussianization from a finite sample size is presented in Section \ref{Sec4}. 

A similar but different multivariate Gaussianization approach was suggested by \cite{laparra2011iterative}. Here, the authors propose to replace the computationally costly ICA with a simple random rotation matrix. This way, they abandon the effort of minimizing the first term of (\ref{KLD}), and only shuffle the components so that consequent marginal Gaussianization would further decrease $J(\underline{Z})$. Although this approach takes more iterations to converge to a normal distribution (as in each iteration, only the second term of   (\ref{KLD}) is being minimized), it holds several favorable properties. First, the overall run-time is dramatically shorter, since applying random rotations is much faster then linear ICA. Second, it implies a degree of freedom in choosing the rotation matrix, as the suggested random matrix is just one example of a linear shuffling of the components.  

\subsubsection{Bi-terminal multivariate Gaussianization}  
\label{Bi-terminal multivariate Gaussianization}  
Going back to our problem, we would like to Gaussinize $\underline{U}_*$ and $\underline{V}_*$, the  outcomes of the multivariate ACE procedure described above. Ideally, we would like to do so while refraining (as much as we can) from reducing our objective

\begin{equation}
\label{objective}
\log \left(\frac{\left|C_{[\underline{U}_*,\underline{V}_*] }\right|}{|C_{\underline{U}_*}||C_{\underline{V}_*}|}   \right).
\end{equation}
Following the Gaussianization procedures described in the previous section , we suggest  
an iterative process, where in each iteration we apply a rotation matrix to both vectors, followed by marginal Gaussianization to each of the components of the two vectors.
It is easy to show that (\ref{objective}) is invariant to any full rank linear transformations. However, it may be effected by the (non-linear) marginal Gaussianization of the components of  $\underline{U}_*$ or $\underline{V}_*$ (as described in Theorem \ref{gaussianization_theorem}). 
Therefore, we would like to find rotation matrices that minimize the effect of the consequent marginal Gaussianization step. 
This problem is far from trivial. In fact, due to the complicated nature of the marginal Gaussianization procedure, it is quite impossible to a-priorly minimize the effect of the marginal Gaussianization, without actually applying it and see how it behaves. Therefore, we suggest a stochastic search mechanism, which allows us to construct a ``reasonable" rotation matrix.

Our suggested mechanism works as follow: At each iteration we begin by drawing two random rotation matrices $R_1$ and $R_2$ for the two vectors we are to Gaussianize, just like \cite{laparra2011iterative}. We apply marginal Gaussianization to all the components and evaluate our objective (\ref{objective}). Then, we randomly choose two dimensions and an angle, $\theta$, and construct a corresponding rotation matrix $\tilde{R}$ that rotates the space spanned by the two dimensions in $\theta$ degrees.
We apply  $\tilde{R} \cdot R_1$ to our vector, followed by marginal Gaussianization, and again evaluate (\ref{objective}). If the objective increases we assign  $R_1=\tilde{R} \cdot R_1$. We repeat this process a configurable number of times, for the two vectors we are to Gaussianize.

Notice that our suggested procedure applies a stochastic hill climbing search in each step: it randomly searches for the best rotation matrix by gradually composing ``small" rotation steps (of two dimensions and an angle), as the complete search space is practically infinite. This procedure guarantees the convergence into two multivariate normal vectors, as shown by \cite{laparra2011iterative}, under the reasonable assumption that $R_1$ and $R_2$ do not repeatedly converge to identity matrices.

%Notice that under the assumption that the stochastic hill climbing routine finds the optimal rotation matrices, then our suggested bi-terminal Gaussianization procedure offers an optimal greedy routine, as in each step of the process it takes the best possible action with respect to maximizing the objective (\ref{objective}) while decreasing the KLD from normal distributions.    

As we see in our experiments, the Bi-terminal Gaussianization is superior to naively applying a Gaussianization procedure to each of the vectors separately  (as suggested by \cite{chen2001gaussianization} or \cite{laparra2011iterative}), in all the cases we examine.

\iffalse
Stochastic hill climbing for joint Gaussianization.
Input: max_steps

0. let counter=0
1. randomly draw a rotation matrix R
2. randomly draw 2 dimensions and an angle. construct a corresponding rotation matrix , R_1
3. If R*R_1 improves the objective than R=R*R_1, else counter=counter+1
4. if counter>max_steps return R, else repeat 2
\fi

\subsection{Illustrative examples}
\label{multivariate examples}
We now examine our suggests multivariate approach in different setups. As in the univariate case, we draw samples from a given model and bound from below the mutual information $I(\underline{X}, \underline{Y})$ according to (\ref{basic_inequality}). First, we apply the multivariate ACE procedure  (Section \ref{multivariate_ub}) to achieve an upper bound for our objective. Then, we apply 
separate Gaussianization to ACE's outcome, to attain an immediate lower bound for our objective (Section \ref{Practical multivariate Gaussianization}). Further, we tighten this lower bound by replacing the separate Gaussianization with bi-terminal Gaussianization to ACE's outcome (Section \ref{Bi-terminal multivariate Gaussianization}). Since our multivariate AGCE procedure (Section \ref{multivariate AGCE}) is practically infeasible, we refrain from using it. This would be further justified later in our results, as we see that the gap between the lower and upper bounds is relatively small. In all of our experiments, our benchmark would be a direct separate Gaussianization of $\underline{X}$ and $\underline{Y}$, as an immediate alternative. 

We begin with a simple toy example. Let $\underline{X} \sim N(0,I)$ and $\underline{W}  \sim N(0,I)$ be independent random vectors. Define $\underline{Y}=\underline{X}+\underline{W}$, so that $\underline{X}$ and $\underline{Y}$ are jointly normally distributed. Further,  we ``scramble" $\underline{X}$ and $\underline{Y}$ by applying invertible, yet non-monotonic, transformations to each of them separately. 
We ask that the transformations are invertible to guarantee that the (analytically derived) mutual information is preserved. We further require non-monotonic transformations since marginal Gaussianization is invariant to monotonic functions (see Proposition \ref{prop1}), which would make this experiment too easy. In this experiment,  we multiply all the observations in the range $[-1,1]$ by $-1$. This operation simply mirrors these observations with respect to the origin.

\begin{proposition}
\label{prop1}
Let $\tilde{X}=g(X)$ be a monotonic transformation on $X \in \mathbb{R}$. Then Gaussinizing $\tilde{X}$ is equivalent to Gaussinizing $X$.
\end{proposition}
\begin{proof}
Let $V=\Phi_N^{-1}\left(F_{{X}}\left({X}\right)\right)$ be the Gaussianization $\tilde{X}$ and $V=\Phi_N^{-1}\left(F_{{X}}\left({X}\right)\right)$ is the Gaussianization of $X$. Assume that $g$ is monotonically increasing. Then, 
$$F_{\tilde{X}}(a)=P(\tilde{X} \leq a)=P(g(X) \leq a)=P(X \leq g^{-1}(a)).$$ Therefore,  $F_{\tilde{X}}\left(\tilde{X}\right)=F_{X}\left(g^{-1}(\tilde{X})\right)=F_{X}(X)$ and $\tilde{V}=V$. An equivalent derivation holds for the monotonically decreasing case. 
\end{proof} 

Before we proceed, it is important to briefly comment on the implications of the finite sample size in our multivariate experiments. The ACE procedure estimates conditional expectations at each of its iterations. This estimation task is known to be quite challenging in a finite sample size regime. \cite{breiman1985estimating} suggest a \textit{k nearest neighbor} estimator which guarantees favorable consistency properties. Unfortunately, this solution is suffers from the curse of dimensionality \citep{hastie2005elements}. Therefore, as the dimension of our problem increases, we cannot turn to ACE and have to settle for suboptimal solutions. In our experiments, we use the kernel CCA \citep{lai2000kernel} as an alternative to ACE when the dimension size is greater than $d=5$. The kernel CCA (KCCA) is a non-linear generalization to the classical CCA which embeds the data in a high-dimensional Hilbert space and applies CCA in that space. It is known to significantly improve the flexibility of CCA while avoiding over-fitting of the data. Notice that other non-linear CCA extensions, such as \textit{Deep CCA} \citep{andrew2013deep} or \textit{nonparametric CCA} \citep{michaeli2015nonparametric}, may also apply as a finite sample size alternative to ACE.

We now demonstrate our suggested approach to the jointly Gaussian model discussed above.  The left plot of Figure \ref{high_d_example} shows the results we achieve for different dimension sizes $d$. The black line on the top is $I(\underline{X}, \underline{Y})$, which can be analytically derived. The red curve with the squares at the bottom is separate Gaussianization of $\underline{X}$ and $\underline{Y}$, which results in a very poor lower bound to the mutual information due to the non-monotonic nature of the transformation that we apply. The green curve with the squares is ACE, while the dashed blue curve is separate Gaussianization of ACE. Finally, the blue line between them is bi-terminal Gaussianization of ACE. As we can see, ACE succeeds in recovering the jointly Gaussian representation of $\underline{X}$ and $\underline{Y}$, which makes further Gaussianization redundant. Unfortunately, for $d >5$ we can no longer apply ACE and turn to KCCA instead. We use a Gaussian kernel with varying parameters to achieve the reported results. Since the KCCA attains a suboptimal representation it is followed by Gaussianization, which further decreases our objective. Here, we notice the improved effect of the bi-terminal Gaussianization, compared with separate Gaussianization. 

Next, we turn to a more challenging exponential model. In this model, each component of $\underline{X}$ and $\underline{W}$ is exponentially distributed with a unit parameter, while all the components are independent of each other. Again, we define $\underline{Y}=\underline{X}+\underline{W}$ so that $\underline{Y}$ is Gamma distributed. This allows us to analytically derive $I(\underline{X}, \underline{Y})$. As before, we apply an invertible non-monotonic transformation to each of the components of $\underline{X}$ and $\underline{Y}$. Notice that this time we mirror the observations in the range $[0,2]$ with respect to $1$. We then apply a linear rotation, so that the components are no longer independent. The plot in the middle of Figure \ref{high_d_example} demonstrates the results we achieve. As before, we notice that separate Gaussianzation of $\underline{X}$ and $\underline{Y}$ preforms very poorly. On the other hand, ACE as well does not succeed in achieving this MI. This means that no Gaussianzation procedure would allow jointly normal representation of $\underline{X}$ and $\underline{Y}$  without losing information (Lemma  \ref{negative_lemma_2}). Still, by applying bi-terminal Gaussianization to ACE's results we are able to capture more than half of the information in the worst case (for  $d=5$, where ACE still applies). As before, we witness a reduction of performance when turning from ACE to KCCA.

Finally, we go back to the multivariate extension of the Gaussian mixture model described in Section \ref{examples_1} and apply our suggested procedures. Again,  we witness the same behavior described in the previous experiments. In addition, our results indicate that in this model, the Gaussian part of the MI is significantly smaller, compared with the exponential model. This further demonstrates the ability of our method to quantify how well an arbitrary distribution may be represented as jointly normal.

\begin{figure}[ht]
\centering
\includegraphics[width =0.9\textwidth,bb= 80 140 720 470,clip]{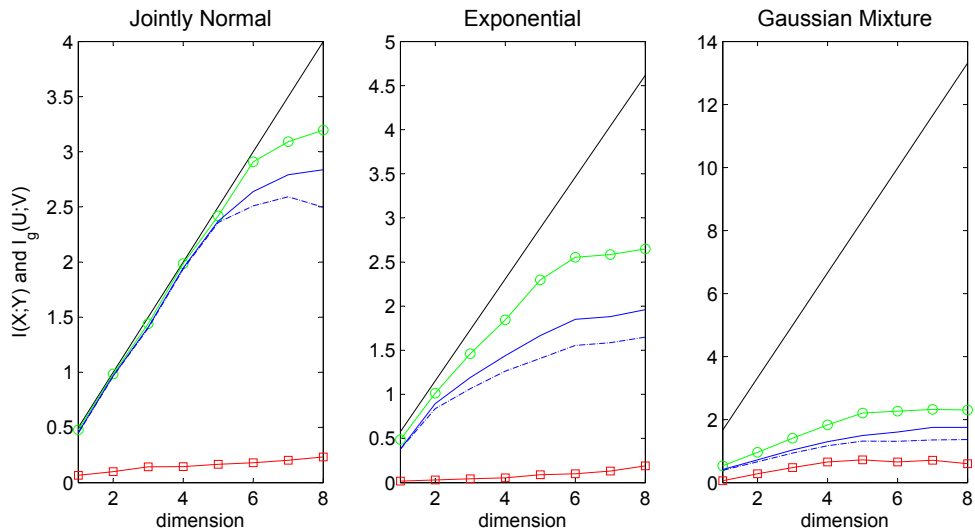}
\caption{Multivariate Gaussianization experiments: The black line on the top of each plot is  $I(\underline{X}, \underline{Y})$. The red curve with the squares at the bottom is separate Gaussianization of $\underline{X}$ and $\underline{Y}$. The green curve with the squares is ACE, while the dashed blue curve is separate Gaussianization of ACE. The blue line in between is bi-terminal Gaussianization of ACE.}
\label{high_d_example}
%\vspace{-1.5em}
\end{figure}

%\section{Finite sample size considerations} \label{Sec4}

\section{Gaussian lower bound for the Information Bottleneck Curve} \label{Gaussian lower bound for the Information Bottleneck Curve}
We now extended our derivation to the Information Bottleneck (IB) curve.
We show that by maximizing the Gaussian lower bound of the mutual information (\ref{basic_inequality}), we allow a maximization of a Gaussian lower bound to the entire IB curve.  We prove this in two steps. First, we show that the IB curve of $\phi(\underline{X}), \psi(\underline{Y})$ bounds from below the IB curve of $X$ and $Y$, for any choice of $\phi, \psi$ (specifically, $\phi(\underline{X}) \sim N$ and $\psi(\underline{Y}) \sim N$, in our case). This property is referred to as the \textit{data processing lemma for the IB curve}. Then, we show that the IB curve of jointly normal random variables bounds from below the IB curve of separately normal random variable. Finally, by applying the GIB \citep{chechik2005information} to the maximally correlated jointly normal random variables that satisfy (\ref{basic_inequality}), we attain the desired Gaussian lower bound for the IB of $\underline{X}$ and $\underline{Y}$.

\begin {lemma}
(data processing lemma for the IB Curve): Denote the maximizer of the IB problem (\ref{IB problem})
\iffalse
 \begin{equation}
\label{DPL for IB} 
\begin{aligned}
& {\max_{T}}
& & I(T(\underline{X}); \underline{Y}) \\
& \text{subject to}
& & I(T(\underline{X}); \underline{X}) \leq I_Y\\
\end{aligned}
\end{equation}
\fi
as $I_*^\beta \left(\underline{X}; \underline{Y}\right)$. Then, $I_*^\beta \left(\underline{X}; \underline{Y}\right) \geq I_*^\beta \left(\phi(\underline{X}); \psi(\underline{Y})\right)$ for any $\phi, \psi$, and with equality iff $I\left(\underline{X}; \underline{Y}\right) = I\left(\phi(\underline{X}); \psi(\underline{Y})\right)$.
\end{lemma}

\begin{proof}
We prove this lemma by showing that $I_*^\beta \left(\underline{X}; \underline{Y}\right) \geq I_*^\beta \left(\underline{X}; \psi(\underline{Y})\right)  \geq I_*^\beta \left(\psi(\underline{X}); \psi(\underline{Y})\right)$.
We start with the first inequality. According to the data processing lemma, we have that  $I(T(\underline{X}); \underline{Y}) \geq I(T(\underline{X}); \psi(\underline{Y}))$. Notice that for convenience,  we emphasize that $T$ is indeed a mapping of $X$ alone.  In addition, since our constraint (\ref{IB problem}) is independent of $Y$, we have that $I_*^\beta \left(\underline{X}; \underline{Y}\right) \geq I_*^\beta \left(\underline{X}; \psi(\underline{Y})\right)$, as expected.   
Second, notice that the IB problem (\ref{IB problem}) may be equivalently written as 
\begin{equation}
\label{DPL for IB 2} 
\begin{aligned}
& {\min_{T}}
& & I(T(\underline{X}); \underline{X}) \\
& \text{subject to}
& &I(T(\underline{X}); \underline{Y}) \geq \tilde{I}_Y\\
\end{aligned}
\end{equation}
Denote the minimizer of (\ref{DPL for IB 2}) as $\bar{I}_*^{\gamma}(\underline{X};\underline{Y})$.  Assume that there exists such $\phi$ that 
\begin{equation}
\label{false_assumption}
\bar{I}_*^{\gamma}(\underline{X};\underline{Y})
>\bar{I}_*^{\gamma}(\phi(\underline{X});\underline{Y})
\end{equation}
%which equivalently means that $I^*_{\beta}(\phi(\underline{X});\underline{Y})>I^*_{\beta}(\underline{X};\underline{Y})$. 
This means that for $I(T(\underline{X}); \underline{Y}) \geq \tilde{I}_Y$ and $I(T'(\phi(\underline{X})); \underline{Y}) \geq \tilde{I}_Y$ we have that $I(T(\underline{X}); \underline{X}) >I(T'(\phi(\underline{X})); \phi(\underline{X}))$ where $T$ and $T'$ are the optimizer of (\ref{DPL for IB 2}) with respect to $(\underline{X},\underline{Y})$ and $(\phi(\underline{X}),\underline{Y}$), for a given $\tilde{I}_Y$, respectively.  Let us set $\tilde{T}\equiv T' \circ \phi$ and apply this transformation to $\underline{X}$. Then, we have that the constraint of (\ref{DPL for IB 2}) is met, as $I(\tilde{T}(\underline{X});\underline{Y})\equiv I(T'(\phi(\underline{X}));\underline{Y}) \geq \tilde{I}_Y$. In addition, we have that $$I(\tilde{T}(\underline{X});\underline{X}) \equiv I(T'(\phi(\underline{X}));\underline{X})=I(T'(\phi(\underline{X}));\phi(\underline{X}))$$
where the second equality follows from $T'$ being independent of $\underline{X}$, given $\phi(\underline{X})$.  Therefore, $\tilde{T}=T' \circ \phi$ is a better optimizer to  (\ref{DPL for IB 2}) with respect to $\underline{X}$ and $\underline{Y}$, then $T$. This contradicts the optimality of $T$ as a minimizer to (\ref{DPL for IB 2}), which means that the assumption in (\ref{false_assumption}) is false and our proof is concluded
\end{proof}

\begin {lemma}
Let $\underline{U}$ and $\underline{V}$ be separately Gaussian random vectors with a joint covariance matrix $C_{[\underline{U}, \underline{V}]}$(that is, $\underline{U} \sim N$ and $\underline{V} \sim N$ but $[\underline{U}, \underline{V}]^T$ is not normally distributed). Let $\underline{U}_{jg}, \underline{V}_{jg}$ be two jointly normally distributed random vectors with the same covariance matrix, $C_{[\underline{U}_{jg}, \underline{V}_{jg}]}=C_{[\underline{U}, \underline{V}]}$. Then, the IB curve of $\underline{U}_{jg}$ and $\underline{V}_{jg}$ bounds from below the IB curve of $\underline{U}$ and $\underline{V}$.
\end{lemma}

\begin{proof}
Let $\left(I(\underline{U}_{jg}; \underline{T}), I(\underline{T};\underline{V}_{jg})\right)$ be a point of the IB curve of $\underline{U}_{jg}$ and $\underline{V}_{jg}$. Since $\underline{U}_{jg}$ and $\underline{V}_{jg}$ are jointly normally distributed,  $T$ is necessarily a linear transformation of $\underline{U}_{jg}$, with additive independent Gaussian noise \citep{chechik2005information}. Specifically,  $T=A\underline{U}_{jg}+ \underline{\zeta}$, where $\underline{\zeta} \sim N(0,I)$, independent of $\underline{U}_{jg}$ and $\underline{V}_{jg}$.

Further, let $\underline{T}'=A\underline{U}+\underline{\zeta}$ be the same transformation, applied of $\underline{U}$. Since $\underline{U}$ and $\underline{V}$
are not jointly normal, the point $\left(I(\underline{U}; \underline{T}'), I(\underline{T}';\underline{V})\right)$ is below the IB curve of $\underline{U}$ and $\underline{V}$.
First, notice that 
$$I(\underline{U};\underline{T}') \equiv I(\underline{U};A\underline{U}+\underline{\zeta}) =I(\underline{U}_{jg};A\underline{U}_{jg}+\underline{\zeta})\equiv I(\underline{U}_{jg},\underline{T})$$
where the second equality follows from $\underline{U}$ and $\underline{U}_{jg}$ having the same distribution. In addition, since  $C_{[\underline{U}_{jg}, \underline{V}_{jg}]}=C_{[\underline{U}, \underline{V}]}$ we have that $C_{[A\underline{U}_{jg}+\underline{\zeta}, \underline{V}_{jg}]}=C_{[A\underline{U}+\underline{\zeta}, \underline{V}]}$. Therefore, $I(A\underline{U}+\underline{\zeta}; \underline{V}) \geq  I(A\underline{U}_{jg}+\underline{\zeta}; \underline{V}_{jg}) $, in the same manner as the in (\ref{basic_inequality}). This means that $I(\underline{T}'; \underline{V}) \geq  I(\underline{T}; \underline{V}_{jg}) $. To conclude, we showed that for the two pairs, 
$\left(I(\underline{U}_{jg}; \underline{T}), I(\underline{T};\underline{V}_{jg})\right)$ and $\left(I(\underline{U}; \underline{T}'), I(\underline{T}';\underline{V})\right)$, we have that $I(\underline{U};\underline{T}') = I(\underline{U}_{jg},\underline{T})$ while  $I(\underline{T}'; \underline{V}) \geq  I(\underline{T}; \underline{V}_{jg}) $, as desired.
\end{proof}

The two theorems above guarantee that the IB curve of $\underline{X}$ and $\underline{Y}$ is bounded from below by an IB curve of $\underline{U}_{jg}$ and $\underline{V}_{jg}$, where $C_{[\underline{U}_{jg}, \underline{V}_{jg}]}=C_{[\underline{U}, \underline{V}]}$, and $\underline{U}=\phi(\underline{X}) \sim N$, $\underline{V}=\psi(\underline{Y}) \sim N$. Therefore, in order to maximize this lower bound, one needs to maximize the correlation between $\underline{U}$ and $\underline{V}$, subject to a normality constraint, as discussed through out this manuscript. Moreover, once we have found a pair of ($\underline{U}_{jg}, \underline{V}_{jg}$) with a maximal correlation, we may directly apply the GIB to it, as shown by \cite{chechik2005information}, to achieve the optimal Gaussian lower bound IB curve for $\underline{X}$ and $\underline{Y}$.

\subsection{Examples} 

We now demonstrate our suggested Gaussian lower bound for the IB curve in two different setups. Here, we would like to compare our bound with the ``true" IB curve, and with an additional benchmark off-shelf lower bound. As discussed in Section \ref{intro}, computing the exact IB curve (for a general joint distribution) is not a simple task. This task becomes even more complicated when dealing with continuous random variables. In fact, to the best of our knowledge, all currently known methods provide approximated curves, which do not claim to converge to the exact IB curve. Moreover, these methods fail to provide any guarantees on the extent of their divergence from the true IB curve. Therefore, in our experiments, we apply the commonly used reverse annealing technique \citep{slonim2002information} in order to approximated the ``true" IB curve. The reverse annealing algorithm is initiated by computing the mutual information between $\underline{X}$ and $\underline{Y}$, which corresponds to extreme point where $I_Y \rightarrow \infty$ on the IB curve. Then, $I_Y$ is gradually decreased and the solution of the IB problem (\ref{IB problem}) with the previous value of $I_Y$ serves as a starting point to the currently solved $I_Y$. This results in a greedy ``no-regret" optimization method, which in general, fails to converge to the exact IB curve. However, in some special cases (such as the GIB), it can be shown that the optimal solution for a given $I_Y$ is, in fact, the optimal starting point for a smaller value of $I_Y$. In the general case, it is implicitly assumed to be a reasonable local optimization domain. Since the reverse annealing was originally designed for discrete random variables, we apply discretization (via Gaussian quadratures) to our probability distributions is all of our experiments.

We begin by revisiting the exponential model, described in Section \ref{multivariate examples}. In this model,  $X$ and $W$ are independent exponentially distributed random variables with a unit parameter. We define $Y=X+W$ so that $Y$ is Gamma distributed. As  in Section \ref{multivariate examples} we apply an invertible non-monotonic transformation to  $X$ and $Y$, to make this problem more challenging. Since approximating the  IB curve is involved enough for continuous random variables, we limit our attention to the simplest univariate case.  

The plot on the left of Figure \ref{IB_curves} demonstrates the results we achieve. The black curve on top is the approximated IB curve, using the reverse annealing procedure. The red curve on the bottom is a benchmark lower bound, achieved by simply applying the GIB to $X$ and $Y$, as if they were jointly Gaussian. The blue curve in the middle is our suggested Gaussian lower bound (Section \ref{AGCE}). As we can see, our suggested bound surpasses the GIB quite remarkably. This is mainly due to the non-monotonic transformation we apply, which makes the joint distribution highly non-Gaussian. We further notice that our bound is quite tight for smaller $I_Y$'s (closer to the origin) but increasingly diverges as $I_Y$ increases. The reason is that more compressed representations are more ``degenerate" and are easier to Gaussianize while maintaining reasonably high correlations. 
  
Next, we revisit the more challenging Gaussian mixture model, described in Section  \ref{examples_1}. The right plot in Figure \ref{IB_curves} demonstrates the results we achieve. As before, we notice that our suggested lower bound surpasses the naive benchmark, while demonstrating favorable performance closer to the origin. Comparing the two models, we notice that the Gaussian mixture is more difficult to bound from below using our suggested method. This result is not surprising, given the gap in our ability to bound from below the mutual information in these two models, as discussed in Section \ref{multivariate examples}.

\begin{figure}[!ht]
\centering
\includegraphics[width =0.7\textwidth,bb= 50 200 550 600,clip]{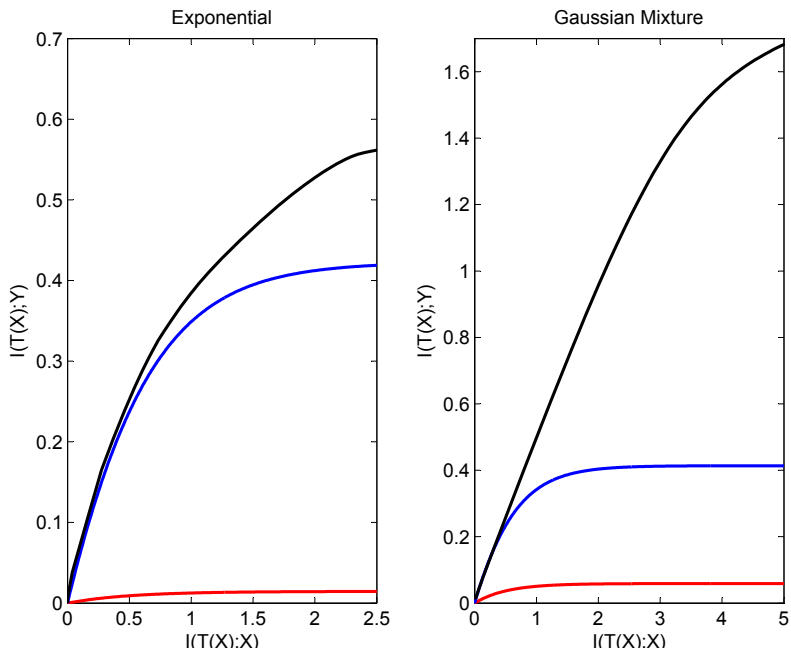}
\caption{IB curves}
\label{IB_curves}
%\vspace{-1.5em}
\end{figure}

\section{Discussion and conclusion}

In this work we address the fundamental problem of normalizing non-Gaussian data, while trying to avoid loss of information. This would allow us to solve complex problems by linear means, as we push information to the data's second moments. We show that our ability to do so is strongly governed by the non-linear canonical correlations of the data. In other words, if the non-linear canonical coefficients of the data fail to maintain its mutual information, then it is impossible to describe its high order dependencies just by second order statistics. This result is of high interest to a broad variety of applications, as solving non-linear problems by linear means is a common alternative in many scientific and engineering fields. Further, we provide a variety of methods to quantify the minimal amount of information that may be lost when normalizing the data. We show that in many cases, our suggested approach is able to preserve a significant portion of the information, even for  highly non-Gaussian joint distributions. Our results improves upon \cite{cardoso2003dependence} information geometry bound, as we show that a tighter bound may be obtained by the AGCE method.

It is important to mention that while our suggested approach is theoretically found, it exhibits several practical limitation in a finite sample-size setup. This is a direct result of our use of the ACE algorithm, which suffers from the curse of dimensionality when applied to high-dimensional data. Therefore, we further examine different non-linear CCA methods, which are less vulnerable to this problem. However, these methods fail to converge to the optimal canonical coefficients. 

Finally, we show that our results may be generalized to bound from below the entire information bottleneck curve. This allows a practical alternative for different approximation methods and  restrictive solutions to the involved IB problem in the continuous case. Our experiments show that the suggested Gaussian lower bound provides a meaningful benchmark to the IB curve, even in highly non-Gaussian setups.

\section{Acknowledgments}
This research was supported by a Fellowship from the Israeli Center of Research Excellence in Algorithms to Amichai Painsky. The authors thank Nori Jacoby for early discussions on the subject.
\bibliographystyle{abbrv}
\bibliography{sigproc}

\end{document}